\documentclass{article}

\pdfoutput=1

\usepackage[preprint]{neurips_2019}
\setcitestyle{square, comma, numbers,sort&compress, super}
\PassOptionsToPackage{numbers, compress}{natbib}
\usepackage{microtype}
\usepackage{hyperref}
\usepackage{url}
\usepackage{natbib}
\usepackage{mathrsfs}
\usepackage{graphicx}
\usepackage{amsmath}
\usepackage{wrapfig}
\usepackage{footnote}
\usepackage[bottom]{footmisc}
\usepackage{amsmath, bm, pdfpages, multirow, threeparttable, amsfonts, booktabs}
\usepackage{algorithm}
\usepackage{algcompatible}
\algnewcommand\algorithmicreturn{\textbf{return}}
\algnewcommand\RETURN{\State \algorithmicreturn}%

\usepackage{setspace}
\usepackage{graphicx}
\usepackage{subfigure}
\usepackage{booktabs} 
\usepackage{microtype}
\usepackage{graphicx}
\usepackage{subfigure}

\usepackage{amsmath,amsfonts,amsthm,amssymb}
\usepackage{mathrsfs}

\usepackage{hyperref}

\newtheorem{theorem}{Theorem}
\theoremstyle{definition}
\newtheorem{definition}{Definition}
\newcommand{\ie}{\emph{i.e.}}
\title{Evolutionary Stochastic Policy Distillation}

\author{%
  Hao Sun, Xinyu Pan, Bo Dai, Dahua Lin, Bolei Zhou\\
  The Chinese University of Hong Kong 
}

\begin{document}
\maketitle

\begin{abstract}
Solving the Goal-Conditioned Reward Sparse (GCRS) task is a challenging reinforcement learning problem due to the sparsity of reward signals. In this work, we propose a new formulation of GCRS tasks from the perspective of the drifted random walk on the state space. Based on the formulation we design a novel method called Evolutionary Stochastic Policy Distillation (ESPD) to solve GCRS tasks by reducing the First Hitting Time of the stochastic process. As a self-imitate approach, ESPD learns a target policy efficiently from a series of its stochastic variants in a supervised learning manner. ESPD follows an Evolution Strategy that applies perturbations upon the action of the target policy as the behavior policy for generating experiences and then select the experiences of the superior stochastic variants. Finally the target policy is updated by policy distillation on the selected experience transitions. The experiments on the MuJoCo robotics control suite show the high learning efficiency of the proposed method.\footnote{Code is available at \url{https://github.com/decisionforce/ESPD}}
 
%
\end{abstract}

\section{Introduction} 
Although Reinforcement Learning (RL) has been applied to various challenging tasks and outperforms human in most cases~\citep{mnih2015human,silver2016mastering,lillicrap2015continuous,vinyals2019grandmaster,pachocki2018openai}, manual effort is still needed to provide sufficient learning signals beside the original \textit{Win or Loss} reward signal~\citep{vinyals2019grandmaster}. In most real-world RL tasks, the rewards are usually extremely sparse.
On the one hand, such reward sparsity hinders the learning process, on the other hand, it also provides the flexibility of learning different policies as different solutions to a certain task in order to get rid of deceptive sub-optimal solutions produced by manually designed rewards~\citep{plappert2018multi}. 

Goal-Conditioned Reward Sparse (GCRS) task is one of the challenging real-world reinforcement learning tasks with extremely sparse rewards. In the task, the goal is combined with the current state as the policy input, while the agent is able to receive a positive reward only when the goal is achieved. In many cases, the GCRS task is also considered as the Multi-Goal task where the goal is not fixed and can be anywhere in the state space. 
Therefore the policy has to learn a general solution that can be applied to a set of similar tasks. 
For example, robotic object grasping is such a GCRS task: the target object could be anywhere on the table, the robot has to adjust its arm to reach the object and grasp it. 
The learning objective of a policy is to find a feasible path from the current state to the goal~\citep{tamar2016value}. Similar tasks include the Multi-Goal benchmarks in robotics control~\citep{plappert2018multi}.

In previous works, reward shaping \citep{ng1999policy}, hierarchical reinforcement
learning~\citep{dietterich2000hierarchical,barto2003recent}, curriculum
learning~\citep{bengio2009curriculum} and learning from demonstrations~\citep{schaal1997learning,atkeson1997robot,argall2009survey,hester2018deep,nair2018overcoming}
were proposed to tackle the challenges of learning through sparse rewards. These approaches provide manual guidance from different perspectives. Besides, the Hindsight Experience Replay (HER)~\citep{kaelbling1993learning,NIPS2017_7090} was proposed
to relabel failed trajectories and assign hindsight credits as complementary to the primal sparse
rewards, which is still a kind of Temporal Difference learning and relies on the value of reward.
Recently the Policy Continuation with Hindsight Inverse Dynamics (PCHID)~\citep{sun2019policy} is proposed
to learn with hindsight experiences in a supervised learning manner, but the learning efficiency is still limited by the explicit curriculum setting.

In this work, we intend to further improve the learning efficiency and stability for these GCRS tasks with an alternative approach based on supervised learning. 
Specifically, by formulating the exploration in GCRS tasks as a random walk in the state space, solving the GCRS task is then equivalent to decreasing the first hitting time (FHT) in the random walk. 
The main idea of our method is encouraging the policy to reproduce trajectories that have shorter FHTs. With such a self-imitated manner, the policy learns to reach more and more \emph{hindsight goals}~\citep{NIPS2017_7090} and becomes more and more powerful to extrapolate its skills to solve the task.
Based on this formulation, we propose a new method for the GCRS tasks, which conforms a self-imitate
learning approach and is independent of the value of rewards. Our agent learns from its own success or
hindsight success, and extrapolates its knowledge to other situations, enabling the learning
process to be executed in a much more efficient supervised learning manner. 
To sum up our contributions:
\begin{enumerate}
    \item By modeling the GCRS tasks as random walks in the state space, we provide a novel Stochastic Differential Equation (SDE) formulation of policy learning and show the connection between policy improvement and the reduction of FHT.
    \item To reduce the FHT from the SDE perspective, we propose Evolutionary Stochastic Policy Distillation (ESPD), which combines the mechanism of Evolution Strategy and Policy Distillation, as a self-imitated approach for the GCRS tasks.
    \item We demonstrate the proposed method on the MuJoCo robotics control benchmark and show our method can work in isolation to solve GCRS tasks with a prominent learning efficiency.
\end{enumerate}


\section{Preliminaries}
\begin{figure*}[t]
\begin{minipage}[htbp]{1.0\linewidth}
			\centering
			\includegraphics[width=1.0\linewidth]{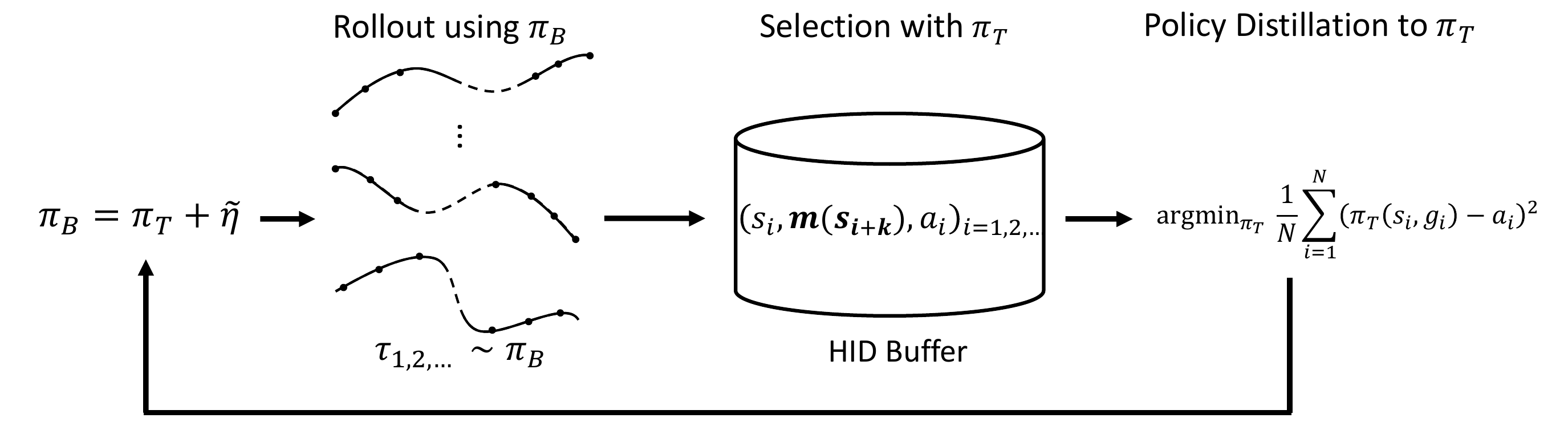}
		\end{minipage}%
\caption{Illustration of Evolutionary Stochastic Policy Distillation. The behavior policy $\pi_B$ is composed of a deterministic policy $\pi_T$ and a stochastic term $\tilde{\eta}$ for exploration. We first generate a batch of trajectories with $\pi_B$, and then we use a SELECT function to select the transitions finished by $\pi_B$ in shorter FHT than $\pi_T$ and store the corresponding HIDs in a buffer. Finally, we improve $\pi_T$ with supervised learning and then use the updated policy to generate new samples.}
\label{fig0}
\end{figure*}
\subsection{Markov Decision Process}
We consider a Markov Decision Process (MDP) denoted by a tuple $\mathcal{M}$ containing:
a state space $\mathcal{S}$, an action space $\mathcal{A}$,
a start state distribution $d(s_0)$, a transition distribution $\mathcal{P}(s_{t + 1}| s_t, a_t)$,
a reward function $r(s)$ and a discount factor $\gamma \in [0, 1]$.
Let $T:\mathcal{S}\times\mathcal{A}\to\mathcal{S}$ model the dynamics 
if the transition is deterministic. Given a policy $\pi(a|s)$, let
$J(\pi) = \mathbb{E}_{\mathcal{M}}[\sum_{t=0}^\infty \gamma^t r(s_t)]$ denote the discounted expected
return, and an optimal policy $\pi^* = \arg \max_\pi J(\pi)$ maximizes that return.

\subsection{Universal Value Function Approximator and Multi-Goal RL}
The Universal Value Function Approximator~\citep{pmlr-v37-schaul15} extends the state space
of Deep Q-Networks~\citep{mnih2015human} to include the goal state as part of
the input, which is useful in the setting where there are multiple goals to achieve. 
Moreover, ~\citet{pmlr-v37-schaul15} show that in such a setting, the learned policy
can be generalized to previous unseen state-goal pairs. Specifically, let $\mathcal{Q} = \mathcal{S} \times \mathcal{G}$
denote the \textit{extended state space} of $\mathcal{M}$ where $\mathcal{G}$ is a goal space.
Since the goal is fixed within an episode, the transition function $T': \mathcal{Q} \times \mathcal{A} \to \mathcal{Q}$ 
on the extended state space can be induced from the original transition $T$ as
$T'((s, g), a) = (T(s, a), g).$
Besides, a representation mapping $m(\cdot): \mathcal{S}\to\mathcal{G}$ is assumed to be known in such 
multi-goal RL frameworks~\citep{plappert2018multi}. Hence, in order to achieve the goal $g$, the agent
must reach a certain state $s_g$ such that $g = m(s_g)$.

We say $\mathcal{M'}$ is a sub-task of $\mathcal{M}$ if $\mathcal{M}'$ is obtained by restricting the start state
distribution $d'(q)$ onto a subset $\mathcal{Q}'$ of the extended state space $\mathcal{Q}$, denoted by 
$\mathcal{M'} = \mathcal{M}_{\mathcal{Q}'} \subset \mathcal{M}$.
In particular, let $\mathcal{M}_{(s, g)}$ denote the sub-task with fixed start state $s$ and goal state $g$.
A partition of $\mathcal{M}$ is a sequence of subtasks $\{\mathcal{M}_k\}_{k=1}^K$ such that
$\mathcal{M}_k \subset \mathcal{M}_{k + 1}$ and $\mathcal{M}_K = \mathcal{M}$.
\subsection{Policy Continuation}
Most multi-goal RL tasks have sparse rewards. In order to motivate the agent to reach the goal efficiently,
the reward function $r(s, g)$ is usually set non-negative if $g = m(s)$ while there is a \emph{negative} penalty  
otherwise. Such a reward distribution can exhibit optimal substructure of the policy.
\begin{definition}{\bfseries Policy Continuation}
Given a policy $\pi$ defined on the sub-task $\mathcal{M}_\mathcal{Q'}$ and a policy $\Pi$ defined on the
sub-task $\mathcal{M}_\mathcal{Q''} \supset \mathcal{M}_\mathcal{Q'}$, we call $\Pi$ is a policy continuation
of $\pi$, if $\Pi(q) = \pi(q),\ \forall q \in \mathcal{Q}'_\pi$,
where $\mathcal{Q}'_\pi$ is the set of all extended states reachable by $\pi$ within task $\mathcal{M}_{\mathcal{Q}'}$

\end{definition}
\begin{theorem}
If $\pi$ is an optimal policy of sub-task $\mathcal{M}_\mathcal{Q'}$,
then there exists an optimal policy $\Pi$ of $\mathcal{M}_\mathcal{Q''} \supset \mathcal{M}_\mathcal{Q'}$
such that $\Pi$ is a policy continuation of $\pi$.
\end{theorem}
\begin{proof}
Let $\Pi'$ be an optimal policy of $\mathcal{M}_\mathcal{Q''}$ and construct $\Pi$ as follows:
\begin{align}
\Pi(q) = \begin{cases}
\pi(q) & \mathrm{if\ } q \in \mathcal{Q}'_\pi\\
\Pi'(q) & \mathrm{otherwise.}
\end{cases}
\end{align}
It is straightforward to see that $\Pi$ is the optimal policy continuation of $\pi$.
\end{proof}
The above theorem is simple yet powerful. It enables the agent to perform supervised learning
from experience replay as long as the trajectory is optimal for some sub-tasks.
Nevertheless, in general given a trajectory $\{(s_0, g), a_0, r_0, \cdots, (s_T, g), a_T, r_T\}$
generated by the policy $\pi$, it is not easy to decide whether $\pi$ is optimal for
the sub-task $\mathcal{M}_{(s_0, m(s_T))}$. If we further assume the negative reward is a \emph{constant} value,
\ie the agent should learn to achieve the goal with minimum actions,
\citet{sun2019policy} proposed to use the partition induced from the $k$-step solvability
to help decide the sub-task optimality.
\begin{definition}{\bfseries $k$-Step Solvability}
Given a sub-task $\mathcal{M}_{(s, g)}$ of a certain system with deterministic dynamics,
we call $\mathcal{M}_{(s, g)}$ is $k$-step solvable with a policy $\pi$ if the goal $g$ can be achieved
from $(s, g)$ within $k$ steps under $\pi$, \ie, set $q_0 = (s, g)$ and
$q_{t + 1} = T'(q_t, \pi(q_t)) = (s_{t + 1}, g)$ for $t \geq 0$, $\exists i \leq k$ such that $m(s_{i}) = g$.
We call a sub-task $\mathcal{M'}$ is $k$-step solvable with $\pi$ if any
$\mathcal{M}_{(s, g)} \subset \mathcal{M}'$ is $k$-step solvable with $\pi$. Specifically, if $\pi$ is the
optimal policy of $\mathcal{M'}$, we simply call $\mathcal{M'}$ is $k$-step solvable.
\end{definition}
Consider the partition $\{\mathcal{M}_k\}_{k=1}^K$ of $\mathcal{M}$ where for any $k < K$, $\mathcal{M}_k$
is the maximal $k$-solvable sub-task. Suppose the agent has learnt the optimal policy $\pi_k$ of $\mathcal{M}_k$,
it can decide whether a trajectory of length $(k + 1)$ is optimal by testing whether the corresponding
sub-task is $k$-step solvable. And those trajectories passing the TEST can serve as supervised training samples
for extending $\pi_k$ to an optimal policy continuation $\pi_{k + 1}$ on $\mathcal{M}_{k + 1}$.

\section{Method}
\subsection{Problem Formulation}

In a given goal-oriented reinforcement learning task, we assume there exists an unknown
metric $\mathcal{D}(s_t, g)$ that represents the distance between the current state $s_t$
and the desired goal state $g$. For example $\mathcal{D}(s,g)$ is the Euclidean distance in
barrier-free navigation tasks; or the Manhattan distance in navigation tasks with obstacles.

A feasible solution of the task should be a policy $\pi(s_t,g)$ that outputs an action
$a_t$, such that the distance $\mathcal{D}(s_{t+1},g)\le\mathcal{D}(s_t,g)$ for deterministic
dynamics, or $\mathbb{E}[\mathcal{D}(s_{t+1},g)]\le\mathbb{E}[\mathcal{D}(s_t,g)]$
for stochastic dynamics. 
We assume $\mathcal{D}(s,g)$ is continuous and differentiable on $s$, and
$\mathrm{d}s = -\xi\frac{\partial \mathcal{D}(s,g)}{\partial s} \mathrm{d}t$ is a feasible move, as it decreases the distance between $s$ and $g$ when $\xi$ is sufficiently small. 
We further assume the state is a vector, the state transition $\Delta s_t = s_{t+1} - s_t$ is determined by both the dynamics $\phi_s(\cdot): \mathcal{S}\times\mathcal{A} \rightarrow \Delta \mathcal{S}$ and the action
$a_t = \pi(s_t,g)$ provided by the policy. We may write a sufficient condition for
a feasible policy:
\begin{equation}
    \phi_s(\pi(s,g)) = -\xi\frac{\partial \mathcal{D}(s,g)}{\partial s} \mathrm{d}t,
\end{equation}
we further assume $\phi_s^{-1}(\cdot)$ exists\footnote[2]{The assumption can be released to a existence of pseudo inverse: $\phi_s^{-1}(\phi_s(a)) \in A'$, where $A'$ is a set s.t. $\forall a'\in A'$, $\phi_s(a') = \phi_s(a)$.}, \ie~$\forall a \in \mathcal{A}$,
$\phi_s^{-1}(\phi_s(a)) = a$. Hence, by parameterizing the policy
$\pi(s,g)$ with $\theta$, we have
\begin{equation}
   \pi_\theta(s,g) =  \phi_s^{-1}\left(-\xi\frac{\partial \mathcal{D}(s,g)}{\partial s}\mathrm{d}t\right),
\end{equation}
is a feasible policy, \ie, it tends to solve the GCRS task as it continuously minimizes the distance between the current state and the goal.
The above equation tells us, in order to learn a well-performing policy, the policy should
learn two unknown functions: the inverse dynamics $\phi_s^{-1}(\cdot)$ and the derivatives
of distance metric $\mathcal{D}(\cdot,\cdot)$ over $s$ and $g$ with regard to the state $s$.
The work of~\citet{sun2019policy} proposed PCHID to use Hindsight Inverse Dynamics (HID)
as a practical policy learning method in such GCRS tasks. Specifically, in Inverse Dynamics,
a model parameterized by $\theta$ is optimized by minimizing the mean square error of predicted
action $\hat{a}_t$ and executed action $a_t$ between adjacent states $s_t$ and $s_{t+1}$, \ie
~$\theta = \mathop{\arg\min}_{\theta} (a_t - \hat{a}_t)^2 = \mathop{\arg\min}_{\theta} (a_t - \pi_{\theta}(s_t,s_{t+1}))^2$. The
HID revises the latter $s_{t+1}$ with its goal correspondence $g_{t+1} = m(s_{t+1})$, where the
mapping $m$ is assumed to be known in normal GCRS task settings. $m: {\mathcal{S}}\to \mathcal{G}$
s.t. $\forall s\in{\mathcal{S}}$ the reward function $r(s,m(s)) = 1$. In the single step transition setting, the learning objective of the policy is to learn HID by
\begin{equation}
\label{eq_hid}
    \theta = \mathop{\arg\min}_{\theta} (a_t - \hat{a}_t)^2 = \mathop{\arg\min}_{\theta} (a_t - \pi_{\theta}(s_t,g_{t+1}))^2.
\end{equation}

Eq.\ref{eq_hid} shows the HID can be used to train a policy with supervised learning by minimizing the prediction error. However, to get more capable policy that is able to solve harder cases,
training the policy only with 1-step HID is not enough. The work of PCHID then proposed to check the
optimality of multi-step transitions with a TEST function and learn multi-step HID recursively. Such an explicit curriculum learning strategy is not efficient as multi-step transitions can only be collected after the
convergence of previous sub-policies. 

Here we interpret how PCHID works from the SDE perspective. 
Practically, a policy is always parameterized as a neural network trained from scratch
to solve a given task. At beginning the policy will not be fully goal-oriented as a
feasible policy. With random initialization, the policy network will just perform
random actions regardless of what state and goal are taken as inputs. We use a coefficient
$\epsilon$ to model the goal awareness of the policy, e.g., $\epsilon=0$ denotes a purely random policy, and $\epsilon=1$ denotes a better policy. In order to collect diverse experiences and improve our target policy, we follow traditional RL approaches to assume a random noise term denoted by $N$
with coefficient $\sigma$ to execute exploration. Hence, the behavioral policy becomes:
\begin{equation}
\label{equation_4}
    \pi_{\mathrm{behave}} =\pi_{\theta}(s,g) + \sigma N 
    = \epsilon \phi_s^{-1}\left(-\xi\frac{\partial \mathcal{D}(s,g)}{\partial s}\mathrm{d}t\right) + \sigma N.
\end{equation}

The behavioral policy above combines a deterministic term and a stochastic term, which in
practice can be implemented as a Gaussian noise or OU-noise~\citep{lillicrap2015continuous}. Although we assume a deterministic
policy $\pi_{\theta}(s,g)$ here, the extension to stochastic policies is straightforward, e.g., the network can predict the mean value and the standard deviation of an action to form a Gaussian
policy family and the Mixture Density Networks~\citep{bishop1994mixture} can be used for more powerful policy representations.

With such a formulation, the PCHID can be regarded as a method that \emph{explicitly}
learns the inverse dynamics with HID, and progressively learns the metric $\mathcal{D}(s,g)$
with Policy Continuation (PC). In this work, we justify that the approach can be extended to a more
efficient synchronous setting that \emph{implicitly} learns the inverse dynamics $\phi_s^{-1}(\cdot)$
and the derivatives of distance metric $\mathcal{D}(\cdot,\cdot)$ with regard to state $s$ at the same time. 
The key insight of our proposed method is  \emph{minimizing the First Hitting Time~\citep{alili2005representations} of a drifted random walk} (Eq.\ref{equation_4}).

Concretely, the simplest case of Eq.\ref{equation_4} is navigating in the Euclidean space, 
where the distance metric is $\mathcal{D}(s,g) = ||g-s||_2$ and the transition dynamics is an identical
mapping, \ie~$\phi_s(a) = a \in \mathcal{A}\subset \mathcal{S} $, and by applying a Gaussian
noise on the action space, we have
\begin{equation}
    \pi(s,g) = ds = \epsilon\frac{g-s}{||g-s||_2} \mathrm{d}t + \sigma \mathrm{d}W_t,
    \label{equation_norm_ou}
\end{equation}
which is a Stochastic Differential Equation. As our learning objective is to
increase the possibility of reaching the goal in a finite time horizon, the problem can
be formulated as minimizing the FHT 
$\tau = \inf\{{t>0: s_t = g|s_0 > g}\}$, \ie~hitting
the goal in the state space. In practice, the goal state is always a region in the
state space~\citep{plappert2018multi}, and therefore the task is to cross the region as
soon as possible.



\subsection{Evolutionary Stochastic Policy Distillation} 
\begin{figure*}[t]
\begin{minipage}[htbp]{1.0\linewidth}
			\centering
			\includegraphics[width=1.0\linewidth]{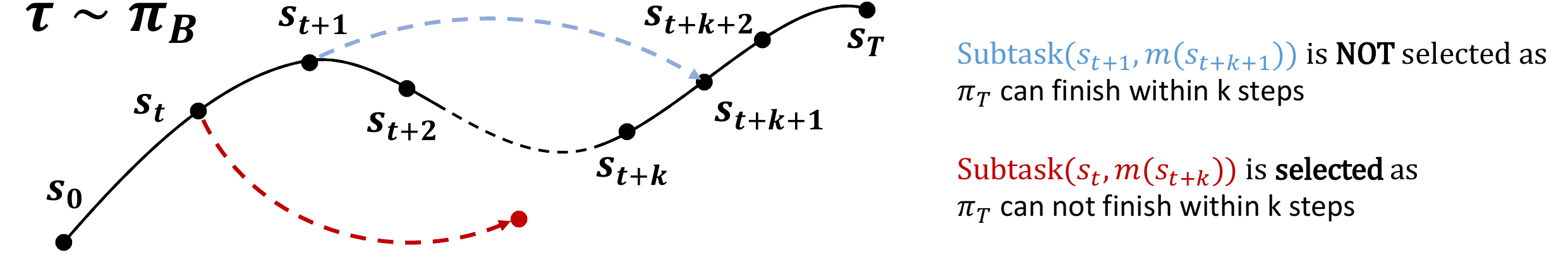}
		\end{minipage}%
\caption{Illustration of the selection process: first we generate episodes with the stochastic behavior policy $\pi_B$, which is composed of the deterministic target policy $\pi_T$ and a noise term drawn from Gaussian, then we check the superiority of generated transitions over the target policy. If $\pi_T$ can not find a shorter path for a transition generated by $\pi_B$, the select function will return True and the transition will be stored for Stochastic Policy Distillation. Therefore, $\pi_T$ will learn to evolve to solve more sub-task, \ie, transitions, continuously.}
\label{illu_fig}
\end{figure*}
Our proposed method combines evolutionary
strategy with policy distillation to minimize FHT. Specifically, ESPD maintains a target deterministic policy
$\pi_T$, parameterized as a policy network, and a behavioral stochastic policy $\pi_B$
\begin{equation}
    \pi_B = \pi_T + \tilde{\eta},  \tilde{\eta}\sim\mathcal{N}(0,\sigma^2),
\end{equation}
according to Eq.\ref{equation_4}, \ie~the behavior policy comes from adding a Gaussian exploration noise upon the target policy $\pi_T$, as in previous deterministic policy learning literature~\citep{lillicrap2015continuous,fujimoto2018addressing}. For the policy update step, ESPD use the evolutionary idea by distilling the well-performed behavior policies, in terms of FHT, to the target policy, instead of applying policy gradient or the zeroth-order approximation of policy gradient~\citep{salimans2017evolution} to the target policy.

Concretely, during training, $\pi_B$ first interacts with the environment and collects a batch of transition samples, permitting us to generate a batch of HIDs, regardless of their optimality. These HIDs contain a set of transition tuples $(s,g',a)$, where $g'$ denotes the hindsight goal. \ie, the starting point, final achieved goal, and the corresponding action are included in each of these transition tuples. From an oracle-perspective, these HIDs can be regarded as generated from a series of \emph{unknown deterministic policies} instead of a known stochastic policy $\pi_B$, each provides a individual solution for the state-goal pair task $(s,g')$. Among these unknown oracle-policies, some are better than our current target policy $\pi_T$ in terms of FHT, which means they are able to solve the certain state-goal pair task in fewer steps, or they are able to solve some sub-tasks while the $\pi_T$ is not. Although we are not able to access these well-performing oracle-policies directly, we can distill the useful knowledge from them to $\pi_T$ through their corresponding HIDs.

In practice, we use a SELECT function to distinguish those HIDs that outperform $\pi_T$ and store them in a buffer $\mathcal{B} = \{(s_i,g'_i,a_i)\}_{i=1,2,...},$. The SELECT function can be implemented in different ways, (1) reset the environment to a given previous state, which is always tractable in simulation~\citep{nair2018overcoming}, (2) use classifiers, dynamic models or heuristics~\citep{sun2019policy}. In this work we adopt (1) and leave the usage of model-based SELECT functions to the future work. 
To implement (1), the SELECT function takes in an episode generated by $\pi_B$. Suppose the episode $(s_t,a_t,s_{t+1},a_{t+1},...,s_{t+k})$ is of length $k$, the SELECT function resets environment to the starting state of this episode $s_t$ and runs $\pi_T$ for up to $k$ steps, trying to reach the final achieved state $s_{t+k}$. \ie, at every step, an action of $\pi_T(s, m(s_{t+k}))$ is performed. If $\pi_T$ is \textbf{NOT} able to reach $s_{t+k}$ within $k$ steps, the corresponding transition tuple $(s_t,m(s_{t+k}),a_t)$ will be collected in the buffer $\mathcal{B}$ and $\pi_T$ will learn from these tuples later. Such a procedure is illustrated in Fig.\ref{illu_fig}.

Then, we can apply Stochastic Policy Distillation (SPD) to distill the knowledge from the well-performing oracle-policies to $\pi_T$ so that $\pi_T$ may evolve to be more capable to
tackle the same sub-tasks. To be specific, we use supervised learning to minimize the difference between the action stored in the HID buffer and the action $\pi_T$ predicted. The SPD is conducted as
\begin{equation}
\label{eq_sl}
    \pi_T = \mathop{\arg\min}_{\pi_T} \frac{1}{N}\sum_{i=1}^N(\pi_T(s_i,g'_i) - a_i)^2,
\end{equation}
where $(s_i, g'_i, a_i)$ are sampled from the HID buffer $\mathcal{B}$.
From this point of view, the ESPD method
is composed of evolution strategy and policy distillation, where a stochastic behavior policy $\pi_B$ acts as the perturbation on the action space and produces diverse strategies (\emph{a population}), and we choose those well-performed strategies to distill their knowledge
into $\pi_T$ (a \emph{selection}). Fig.\ref{fig0} provides an illustration of the learning pipeline and Algorithm \ref{Algorithm1} presents the detailed learning procedure of ESPD.

\begin{algorithm}[tbp]
\caption{ESPD}
\label{Algorithm1}
\begin{algorithmic}
	\STATE \textbf{Require}
	\begin{itemize}
		\item a target policy $\pi_T(s,g)$ parameterized by neural network: $\pi_T(s,g) = \pi_{\theta}(s,g)$
		\item a reward function $r(s,g) = 1$ if $g = m(s)$ else $0$
		\item a buffer for ESPD $\mathcal{B}$
		\item a Horizon list $\mathcal{K} = [1,2,...,K]$
		\item a noise e.g.,$\mathcal{N}(0,\sigma^2)$
	\end{itemize}
	\STATE Initialize  $\pi_T$, $\pi_B = \pi_T + \tilde{\eta},  \tilde{\eta}\sim\mathcal{N}(0,\sigma^2)$, $\mathcal{B}$
	\FOR{episode $= 1,M$}
	    \STATE Generate $s_0$, $g$ by the environment
	    \FOR{$t = 0,T-1$}
	        \STATE Select an action by the behavior policy $a_t = \pi_B(s_t,g)$    
			\STATE Execute the action $a_t$ and get the next state $s_{t+1}$
		\ENDFOR
		\FOR{$t = 0, T-1$}
			\FOR{$k = 1, K$}
			    \STATE Calculate additional goal according to $s_{t+k}$ by $g' = m(s_{t+k})$
			    \IF{ SELECT($s_t,g'$) = True}
			        \STATE Store $(s_t,g',a_t)$ in $\mathcal{B}$
			    \ENDIF  
		    \ENDFOR
		\ENDFOR
		\STATE Sample a minibatch $B$ from buffer $\mathcal{B}$
		\STATE Optimize target policy $\pi_T(s_t,g')$ to predict $a_t$ according to Eq.\ref{eq_sl}
		\STATE Update behavior policy $\pi_B = \pi_T + \tilde{\eta},  \tilde{\eta}\sim\mathcal{N}(0,\sigma^2)$
	\ENDFOR
\end{algorithmic}
\end{algorithm}




\section{Experiments}
\begin{figure}[t]
\begin{minipage}[htbp]{0.33\linewidth}
			\centering
			\includegraphics[width=0.85\linewidth]{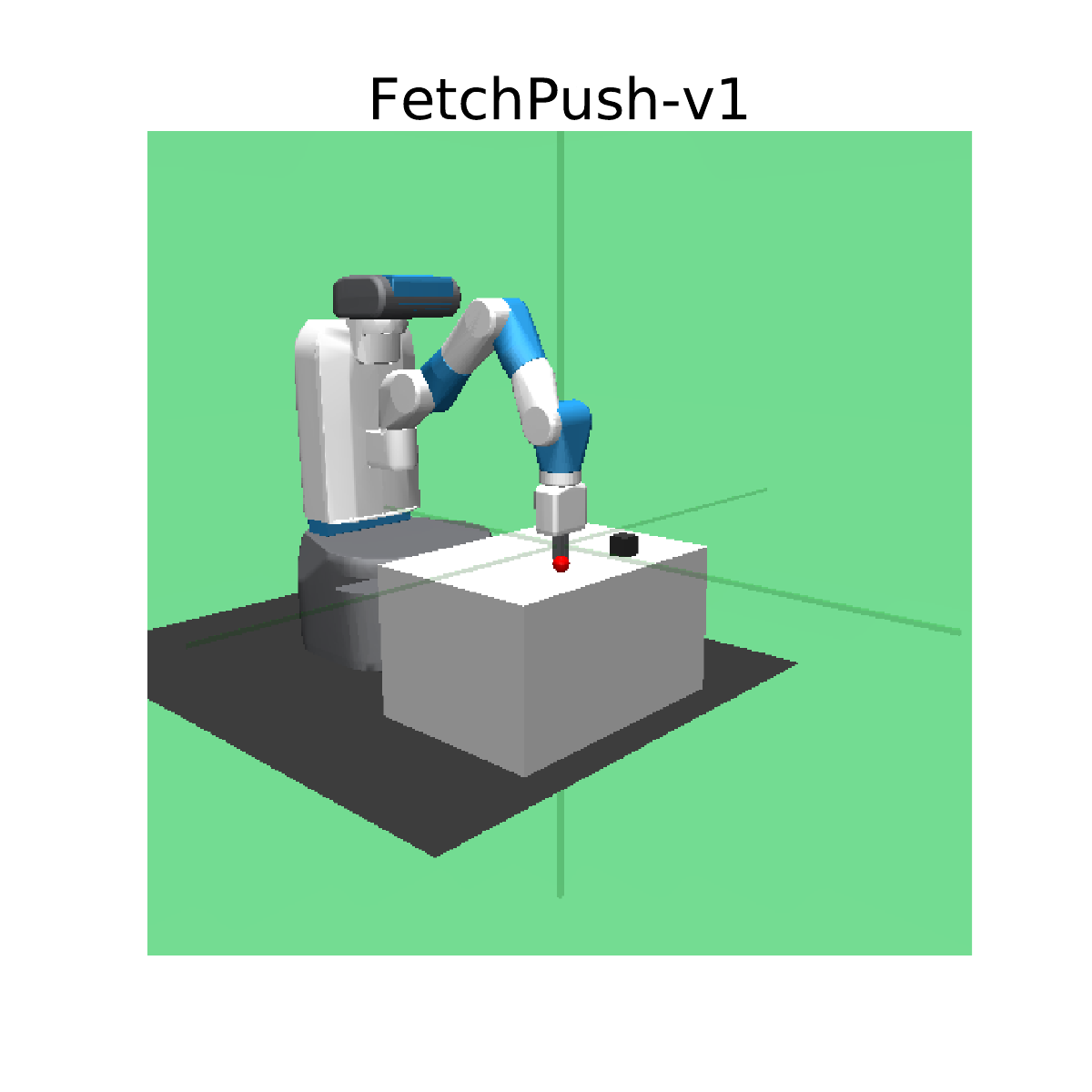}
		\end{minipage}%
		\begin{minipage}[htbp]{0.33\linewidth}
			\centering
			\includegraphics[width=0.85\linewidth]{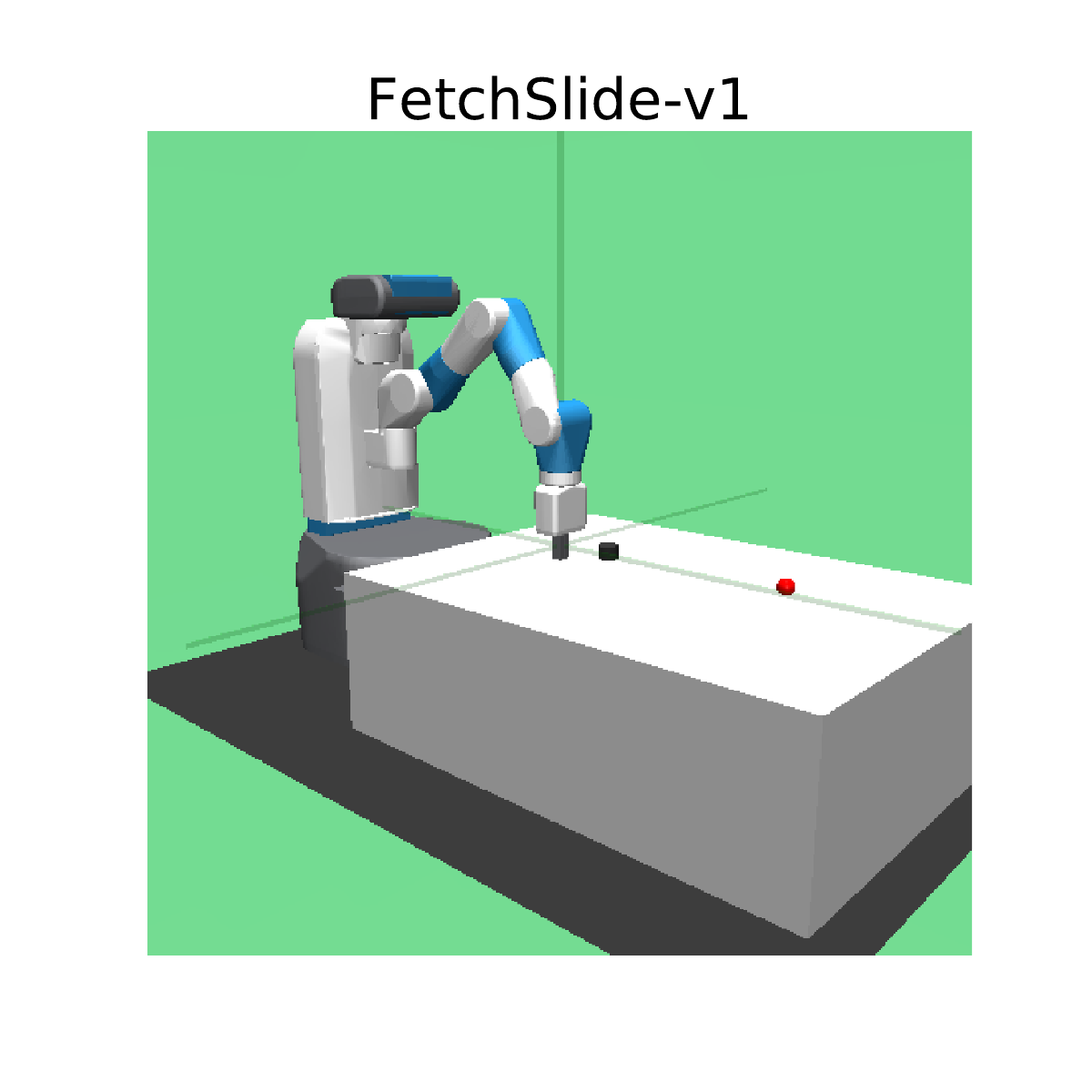}
		\end{minipage}%
		\begin{minipage}[htbp]{0.33\linewidth}
			\centering
			\includegraphics[width=0.85\linewidth]{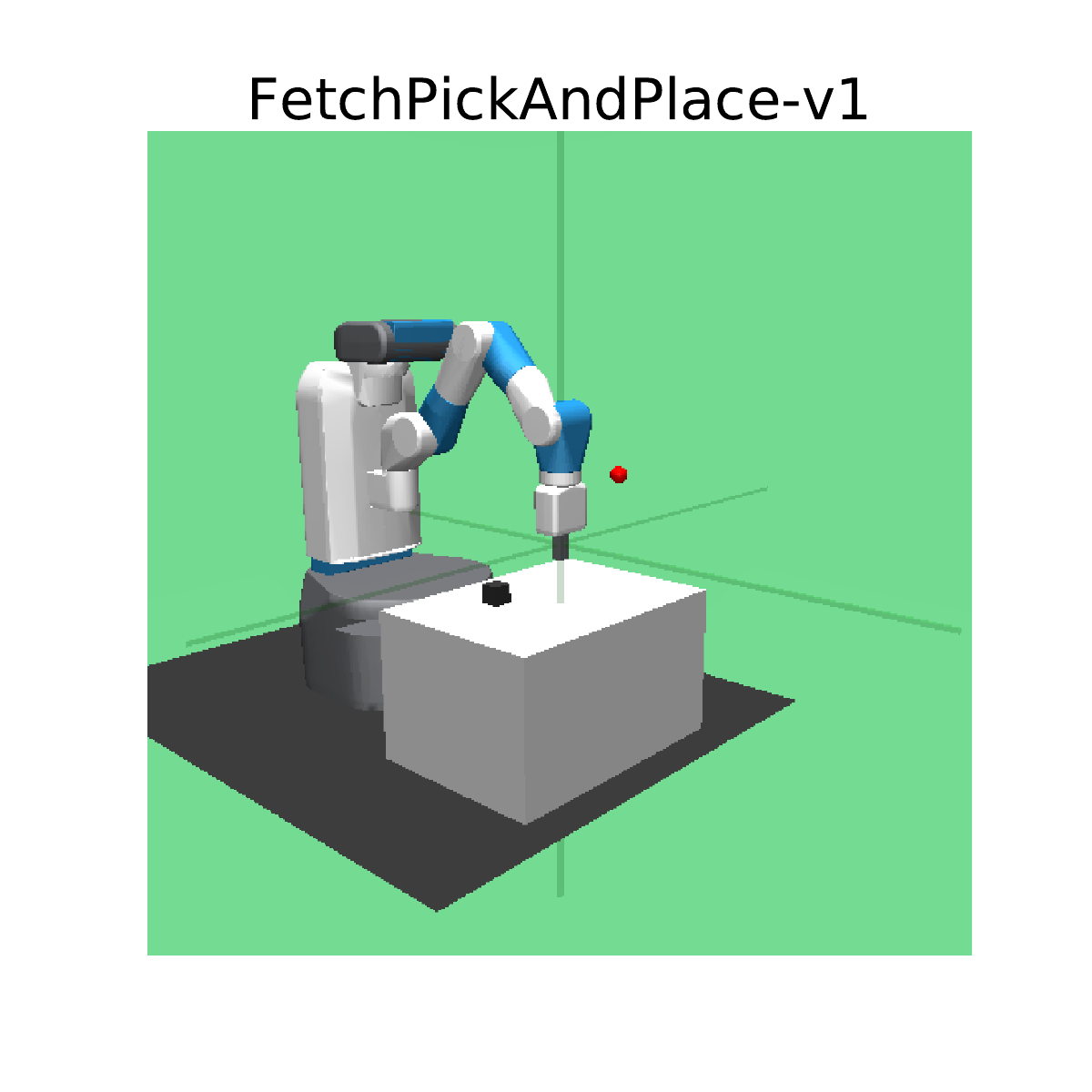}
		\end{minipage}%
\caption{Three robotic manipulation environments: FetchPush, FetchSlide and FetchPickAndPlace.}
\label{fig_env}
\end{figure}
\subsection{Result on the Fetch Benchmarks}
We demonstrate the proposed method on the Fetch Benchmarks. Specifically,
we evaluate our method on the FetchPush, FetchSlide and
FetchPickAndPlace environments, as shown in Fig.\ref{fig_env}. We compare our proposed method
with the HER~\citep{NIPS2017_7090,plappert2018multi} released in OpenAI
Baselines~\citep{baselines} and the Evolution Strategy~\citep{salimans2017evolution} which is a counterpart of our method with parameter noise. As PCHID~\citep{sun2019policy} can be regarded as a special case of ESPD if we gradually increase the hyper-parameter Horizon in ESPD from $1$ to $K$, the performance of PCHID is upper-bounded by ESPD and we do not include it as a baseline. Such result can be inferred from our ablation study on the Horizon $K$ in the next section, which shows smaller $K$ limits the performance, and achieves worse learning efficiency than $K=8$, the default hyper-parameter used in ESPD\footnote{In FetchSlide, our ablation studies in the next section shows $K=12$ outperforms $K=8$.}.

Fig.\ref{main_results} shows the comparison of different approaches. For each environment, we conduct 5 experiments with different random seeds and plot the averaged learning curve. 
Our method shows superior learning efficiency and can learn to solve the task in fewer episodes in all the three environments.

\begin{figure*}[t]
\begin{minipage}[htbp]{0.33\linewidth}
			\centering
			\includegraphics[width=1.0\linewidth]{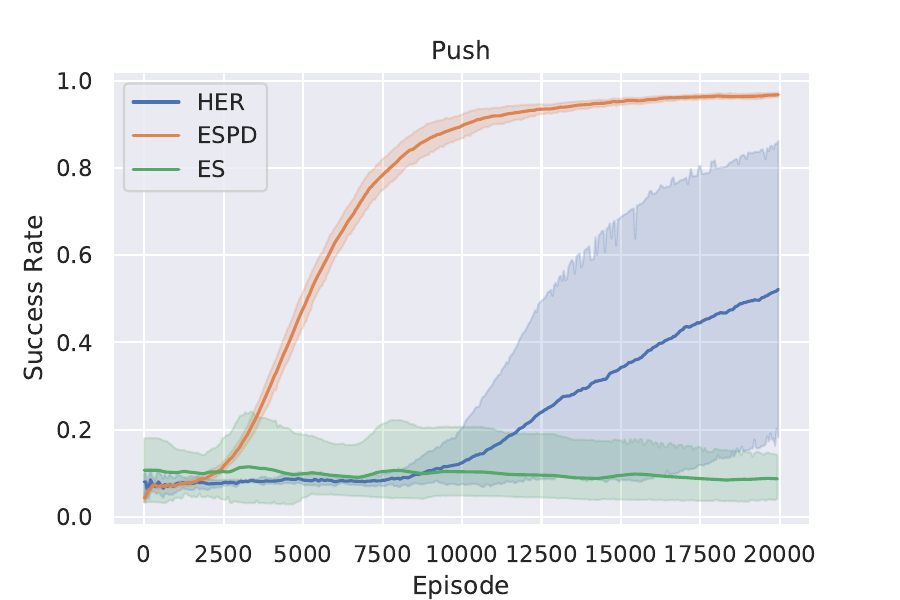}
		\end{minipage}%
		\begin{minipage}[htbp]{0.33\linewidth}
			\centering
			\includegraphics[width=1.0\linewidth]{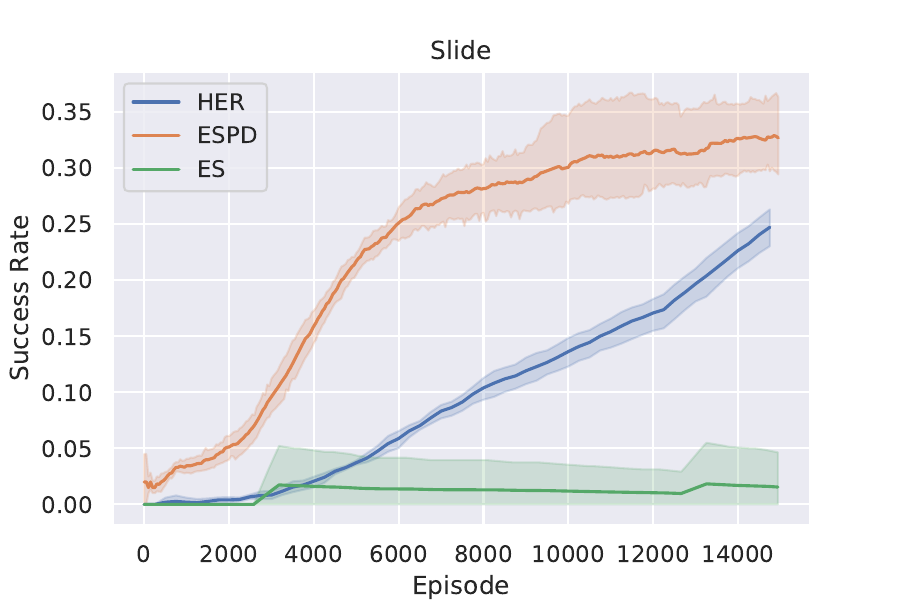}
		\end{minipage}%
		\begin{minipage}[htbp]{0.33\linewidth}
			\centering
			\includegraphics[width=1.0\linewidth]{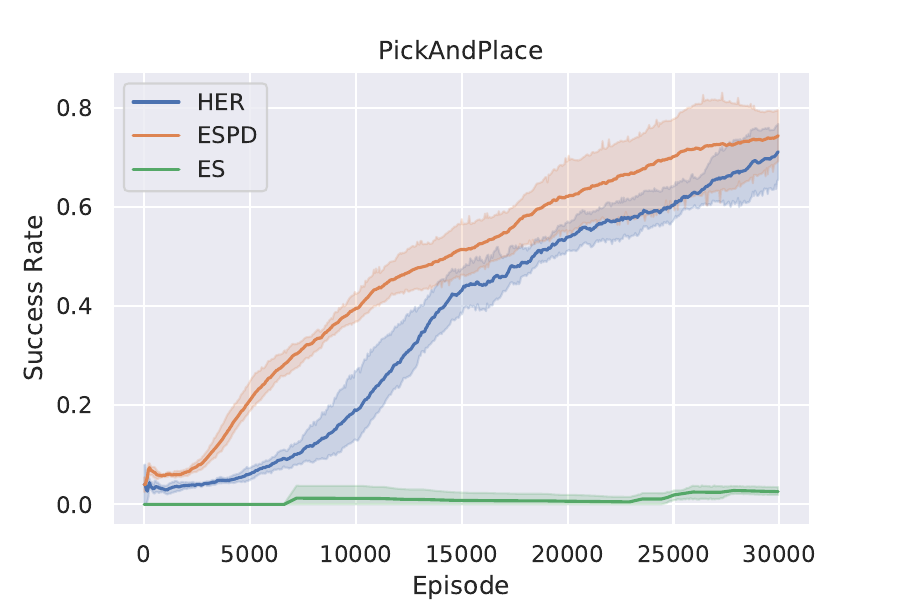}
		\end{minipage}%
\caption{The test success rate comparison on the FetchPush-v1, FetchSlide-v1 and FetchPickAndPlace-v1 among our proposed method (ESPD), HER and Evolution Strategy (ES).}
\label{main_results}
\end{figure*}

\subsection{Ablation Studies}
\paragraph{Exploration Factor}
\begin{figure*}[t]
\begin{minipage}[htbp]{0.33\linewidth}
			\centering
			\includegraphics[width=1.0\linewidth]{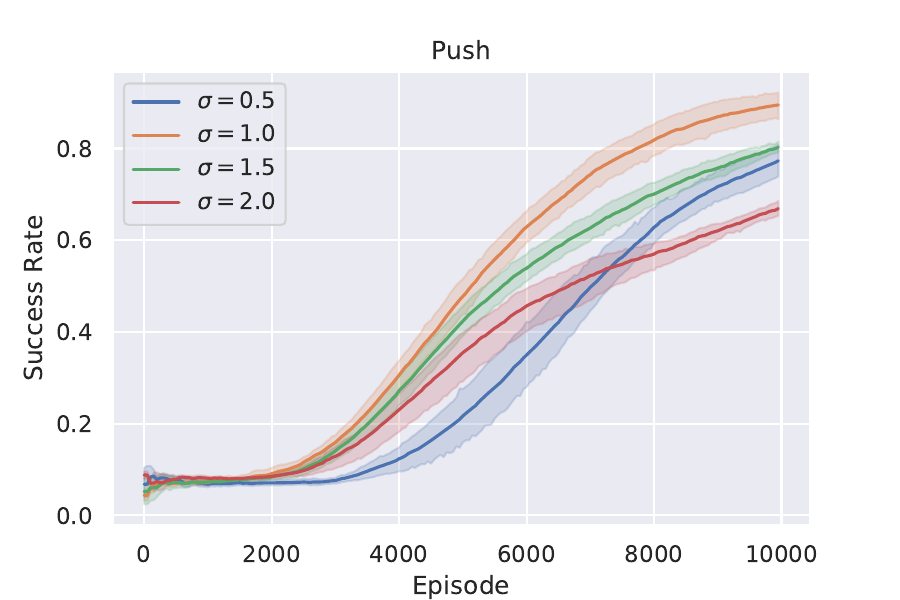}
		\end{minipage}%
		\begin{minipage}[htbp]{0.33\linewidth}
			\centering
			\includegraphics[width=1.0\linewidth]{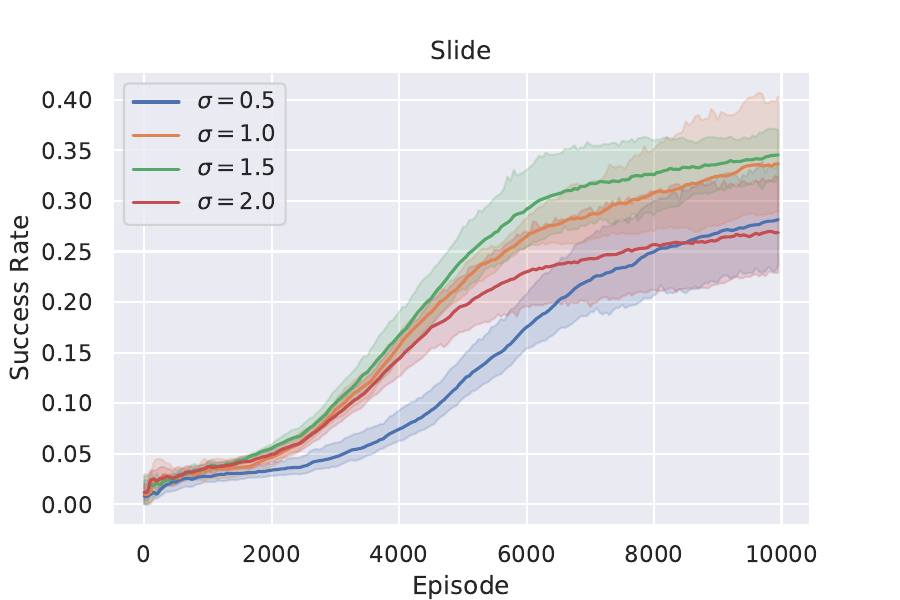}
		\end{minipage}%
		\begin{minipage}[htbp]{0.33\linewidth}
			\centering
			\includegraphics[width=1.0\linewidth]{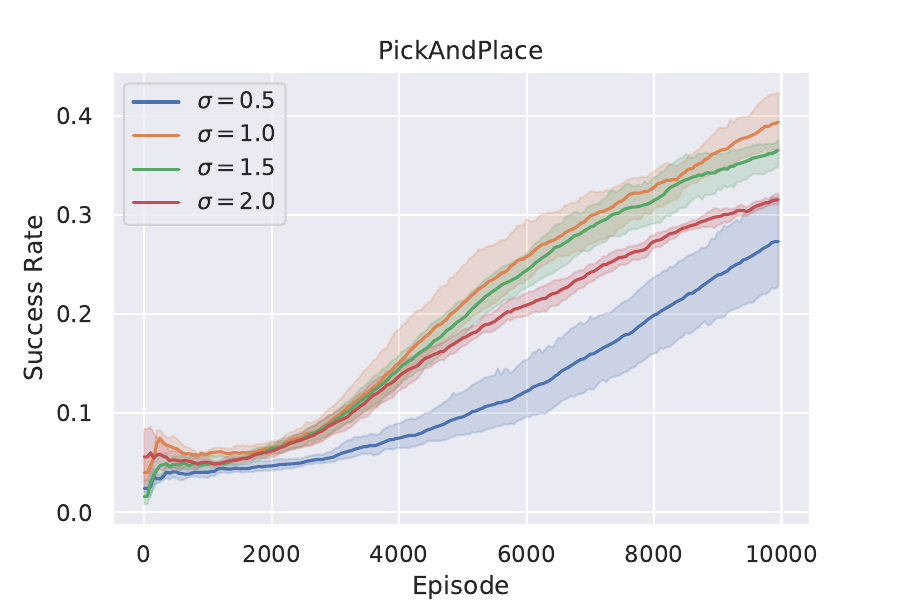}
		\end{minipage}%

\caption{The test success rate comparison on the FetchPush-v1, FetchSlide-v1 and FetchPickAndPlace-v1 with different scale of exploration factors. Experiments are repeated with 5 random seeds.}
\label{exp_factor}
\end{figure*}

The exploration factor $\sigma$ controls the randomness of behavior policy and therefore determines the behavior of generated samples. While larger $\sigma$ helps the agents to benefit exploration by generating samples with large variance, smaller $\sigma$ helps to generate a biased sample with little variance. Here we need to select a proper $\sigma$ to balance the variance and bias.
Fig.\ref{exp_factor} shows our ablation study on the selection of different exploration factors. The results are generated with 5 different random seeds. 
We find in all environments, the exploration factor $\sigma = 1$ provides sufficient exploration and relatively high learning efficiency.

\paragraph{Horizon $K$}
In our proposed method, the parameter of Horizon $K$ determines the maximal length of sample trajectories the policy can learn from. 
Intuitively, smaller $K$ decreases the learning efficiency as the policy is limited by its small horizon, making it hard to plan for the tasks that need more steps to solve. 
On the other hand, larger $K$ will provide a better concept of the local as well as global geometry of the state space, and thus the agent may learn to solve more challenging tasks. 
However, using large $K$ introduces more interactions with the environment, and needs more computation time.
Moreover, as the tasks normally do not need lots of steps to finish, when the Horizon is getting too large, more noisy actions will be collected and be considered as better solutions and hence impede the learning performance.
Fig.\ref{horizon} shows our ablation studies on the selection of Horizon $K$. The results are generated with 5 different random seeds. We find that $K = 8$ provides satisfying results in all of the three environments. 

\begin{figure*}[t]
\begin{minipage}[htbp]{0.33\linewidth}
			\centering
			\includegraphics[width=1.0\linewidth]{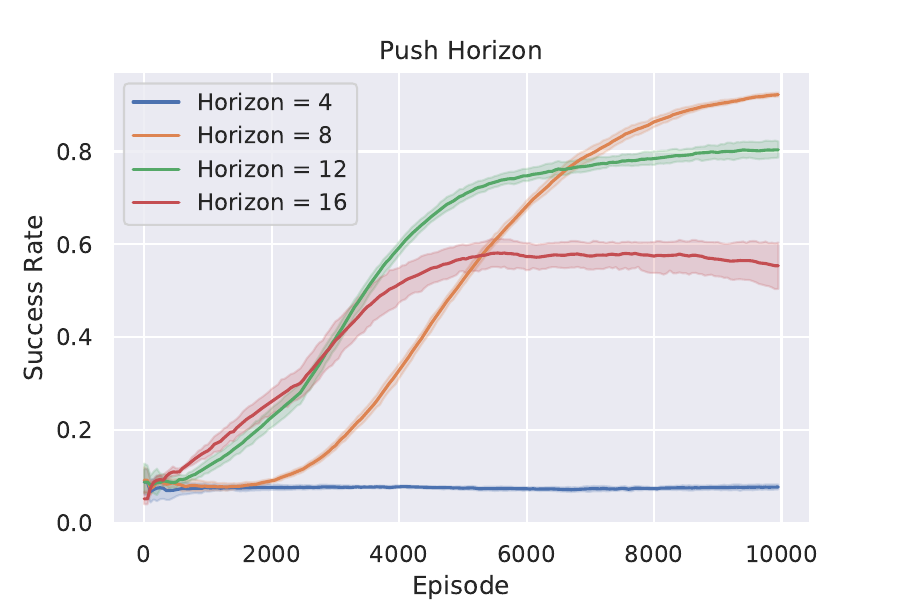}
		\end{minipage}%
		\begin{minipage}[htbp]{0.33\linewidth}
			\centering
			\includegraphics[width=1.0\linewidth]{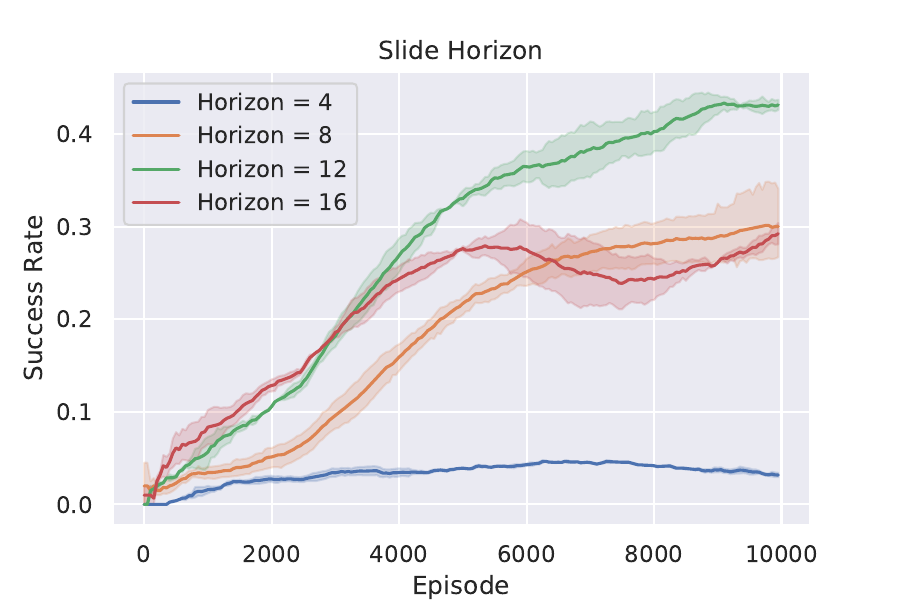}
		\end{minipage}%
		\begin{minipage}[htbp]{0.33\linewidth}
			\centering
			\includegraphics[width=1.0\linewidth]{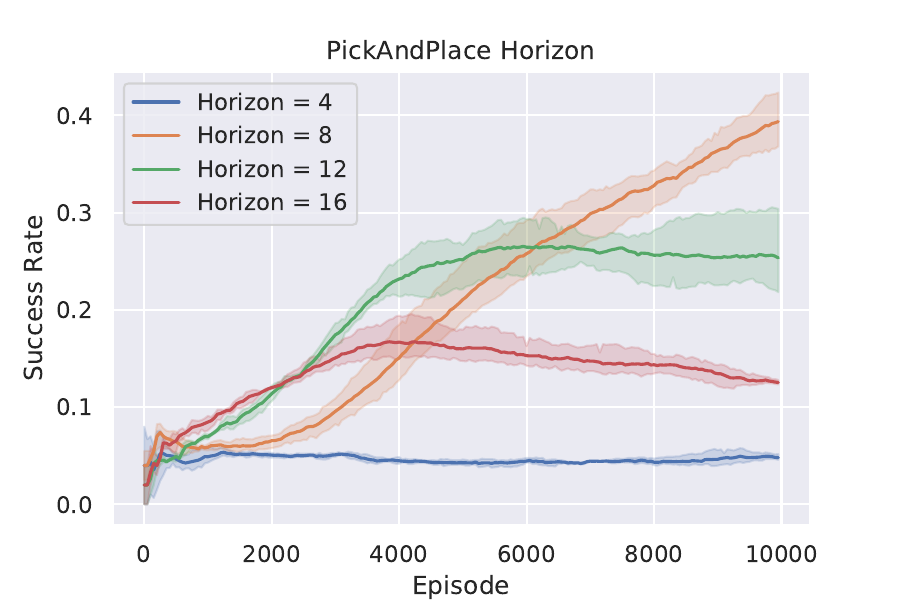}
		\end{minipage}%
\caption{The test success rate comparison on the FetchPush-v1, FetchSlide-v1 and FetchPickAndPlace-v1 with different scales of Horizon $K$. The results are generated with 5 different random seeds}
\label{horizon}
\end{figure*}


\section{Related Work}

\paragraph{Learning with Experts and Policy Distillation}
Imitation Learning (IL) approaches introduce expert data in the learning of a
agent~\citep{pomerleau1991efficient,ross2011reduction}, while similar techniques
are used in the literature of Learning from Demonstrations
(LfD)~\citep{atkeson1997robot,schaal1997learning,argall2009survey}, where experience
of human expert will be collected to help the learning of an agent. Those methods are
further extended in the setting of Deep Q-learning~\citep{mnih2015human,hester2018deep},
combined with DDPG~\citep{lillicrap2015continuous,nair2018overcoming} or to learn
from imperfect expert data~\citep{gao2018reinforcement}.

Policy Distillation was proposed to extract the policy of a trained RL agent with
a smaller network to improve the efficiency as well as the final performance or
combine several task-specific agents together~\citep{rusu2015policy}. Latter extensions
proposed to improve the learning efficiency~\citep{schmitt2018kickstarting}, enhance
multi-task learning~\citep{teh2017distral,arora2018multi}. 

All of those methods start from a trained expert agent or human expert experience
that can solve a specific task~\citep{czarnecki2019distilling}. As a comparison, our proposed method focus on extracting knowledge from stochastic behaviors, which is capable to act as a feasible policy itself with regard to the primal task.

\paragraph{Evolution Strategies and Parameter Noise}
The Evolution Strategy (ES) was proposed by~\citet{salimans2017evolution} as an alternative
to standard RL approaches, where the prevailing temporal difference based value function
updates or policy gradient methods are replaced as perturbations on the parameter space
to resemble the evolution. Later on, \citet{campos2018importance} improved the efficiency
of ES by means of importance sampling. Besides, the method was also extended to be combined
with Novelty-Seeking to further improve the performance~\citep{conti2018improving}.

Thereafter, \citet{plappert2017parameter} proposed to use Parameter Noise as an alternative
to the action space noise injection for better exploration. They show such a perturbation
on the parameter space can be not only used for ES methods, but also collected to improve the sample efficiency by combining it with traditional RL methods. 

While previous ES algorithms apply perturbations on the parameter noise and keep the
best-performed variates, our approach implicitly execute the policy evolution by distilling
better behaviors, therefore our approach can be regarded as an Evolutiaon Strategy based on
action space perturbation.

\paragraph{Supervised and Self-Imitate Approaches in RL}
Recently, several works put forward to use supervised learning to improve the stability and
efficiency of RL. \citet{zhang2019policy} propose to utilize supervised learning to tackle the
overly large gradients problem in policy gradient methods. In order to improve sample efficiency,
the work chose to first design a target distribution proposal and then used supervised
learning to minimize the distance between present policy and the target policy distribution.
The Upside-Down RL proposed by \citet{schmidhuber2019reinforcement} used supervised
learning to mapping states and rewards into action distributions, and therefore acted as a normal
policy in RL. Their experiments show that the proposed UDRL method outperforms several
baseline methods~\citep{srivastava2019training}. In the work of~\citet{sun2019policy}, a curriculum learning
scheme is utilized to learn policies recursively.
The self-imitation idea relevant to ESPD is also discussed in the concurrent work of~\citet{ghosh2019learning}, but ESPD further uses the SELECT function to improve the quality of collected data for self-imitation learning.


\section{Conclusion}

In this work we focus on developing a practical algorithm that can evolve to solve the GCRS problems by distilling knowledge from a series of its stochastic variants. The key insight behind our proposed method is based on our SDE formulation of the GCRS tasks: such tasks can be solved by learning to reduce the FHT. Our experiments on the OpenAI Fetch Benchmarks show that the proposed method, Evolutionary Stochastic Policy Distillation, has high learning efficiency as well as stability with regard to two baseline methods, the Evolution Strategies and Hindsight Experience Replay.

\newpage

\appendix
\section{On the Selection of Exploration Factor}
\begin{figure}[t]
\begin{minipage}[htbp]{0.5\linewidth}
			\centering
			\includegraphics[width=1.0\linewidth]{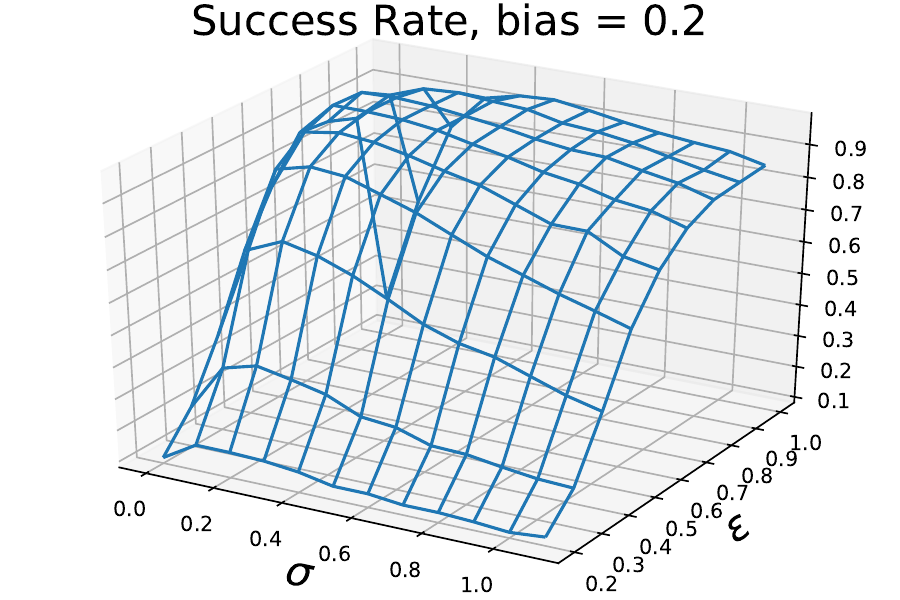}
		\end{minipage}%
		\begin{minipage}[htbp]{0.5\linewidth}
			\centering
			\includegraphics[width=1.0\linewidth]{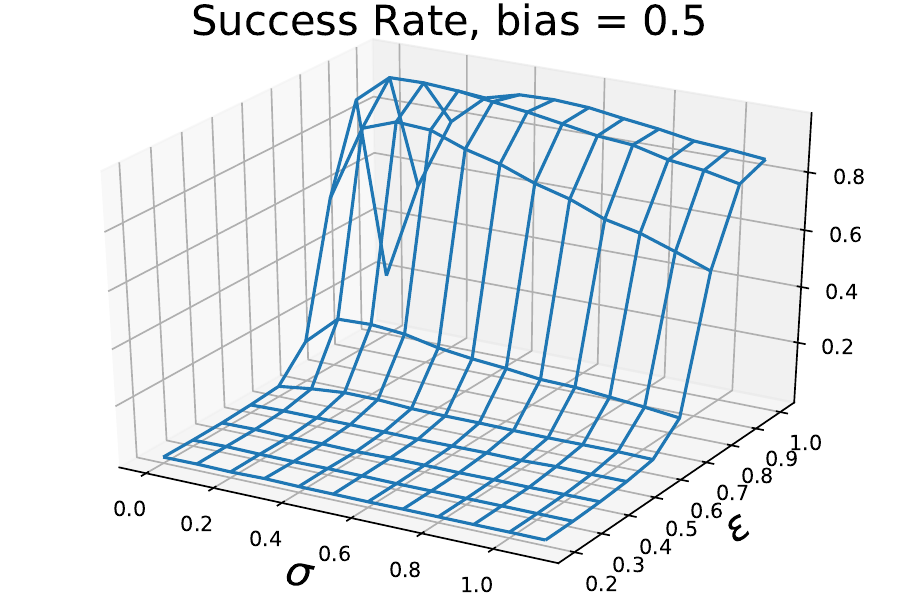}
		\end{minipage}%
\caption{The numerical result with different values of bias $b$ and goal awareness factor $\epsilon$.}
\label{fig_numerical}
\end{figure}

In Algorithm 1, the behavior policy $\pi_B$ is composed of the deterministic target policy
$\pi_T = \pi_{\theta}(s,g)$, and an exploration term $\mathcal{N}(0,\sigma^2)$. During the learning process,
$\pi_T$ provides a biased estimation of the feasible policy $\pi_{\theta}(s,g)$ without variance,
while the exploration term $\mathcal{N}(0,\sigma^2)$ provides unbiased exploration with variance.
The Eq.\ref{equation_4} becomes
\begin{equation}
\label{equation_biased_estimation}
    \pi_{B} = {\pi_{\theta}(s,g)} +  \tilde{\eta}
    = \epsilon \phi_s^{-1}\left(-\xi\frac{\partial \mathcal{D}(s,g)}{\partial s} + b\right) +  \tilde{\eta},  \tilde{\eta}\sim\mathcal{N}(0,\sigma^2),
\end{equation}
where $b$ is an unknown bias introduced by function approximation error or
extrapolation error~\citep{fujimoto2018off} due to the limited number of samples in Algorithm 1.

Intuitively, a large exploration factor, \ie large $\sigma$, will lead to better exploration
thus can help reduce the bias introduced by the $\pi_T$, while smaller $\sigma$ can
reduce the variance, but expose the bias. This is exactly the dilemma of Exploration-Exploitation
(E\&E)~\citep{sutton1998introduction}. We further introduce an effective annealing
method to adjust $\sigma$, the exploration factor to tackle the E\&E challenge.

In the following section, we provide analysis and numerical experiment result based on the special case we have mentioned above to interpret how the exploration factor helps to correct the bias.

\subsection{Revisit the Special Case: Navigation in the Euclidean Space}

At the beginning of learning, the policy $\pi_T$ is initialized randomly. The only way to cross the target region at this moment is to utilize large exploration term, \ie
with large $\sigma$. As the learning continues with limited experience, bias might be introduced into Eq.\ref{equation_norm_ou}
\begin{equation}
    \pi(s,g) = ds = \epsilon \left(\frac{g-s}{||g-s||_2} + b\right)\mathrm{d}t + \sigma \mathrm{d}W_t,
    \label{equation_8}
\end{equation}
where $b$ denotes the bias. One the one hand, such bias may lead to extremely bad solutions if we do not keep a exploration term for bias correction~\citep{fujimoto2018off}. On the other hand, while the policy becomes more capable of navigating in the state space to reach the goal, large exploration term will hinder the agent to step into the goal region. Here we conduct a numerical simulation to show the dependencies of Success Rate, \ie the proportion that successfully hit the goal region in a 2-D Euclidean space
navigation task, on the value of $\sigma$.

\subsection{Numerical Result} According to the previous analysis, the exploration with a random behaved policy in a GCRS task is like a random walk in the state space.
Distinguished from the well known fact that a drunk man will find his way home, but a drunk bird may get lost forever~\citep{kakutani1944143,bracewell1986fourier}, in most cases, the systems we are concerned about have finite boundaries, and the goal, instead of a single point in the state space, always has a non-trivial radius.
Therefore, in known dynamics, we can simulate the behavior of policy at different learning state, e.g., with different bias $b$, goal-awareness $\epsilon$, and investigate how the exploration factor affects the learning outcomes.

Our simulation is based on a bounded region of size $N\times N$, for each episode, a current state and goal are generated randomly in the region. At each timestep, the state updates according to Eq.\ref{equation_8}, with normalized step length. The success rate shows the probability of hitting the goal within a finite time horizon $T$. In our simulation, we apply tabular fixed random bias with different scale (\ie, $U(0,0.2)$ and $U(0,0.5)$), and set $T = 100$, $N = 100$, maximal step length $10$ and goal radius $r = 1$.

The result is shown in Fig.\ref{fig_numerical}. Smaller bias enables success rate increases when the goal-awareness is small. As goal awareness increase, the performance of success rate relies on the selection of exploration factor. For small exploration factors, the performance of biased policy will drastically be hindered, while proper exploration factor value will fix such a problem. Such imperfectness, e.g., the bias is unavoidable when parameterizing the policy with a neural network, hence we maintain a small exploration factor even when evaluating a policy for bias correction. The detailed comparison with different exploration factor in both training and testing phase is discussed in the experiment section.

\subsection{Evaluation Noise}

\begin{figure*}[t]
\begin{minipage}[htbp]{0.33\linewidth}
\centering
\includegraphics[width=1.0\linewidth]{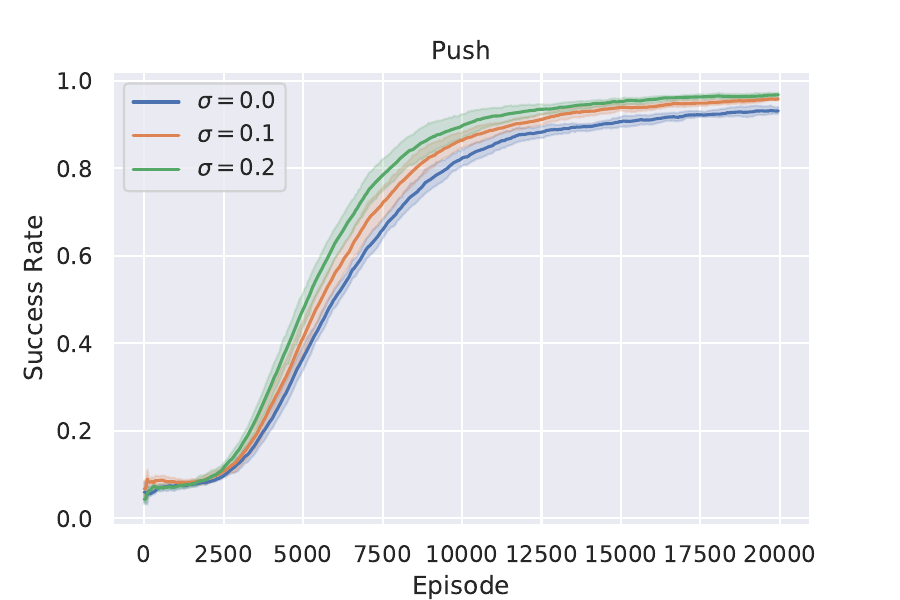}
\end{minipage}%
\begin{minipage}[htbp]{0.33\linewidth}
\centering
\includegraphics[width=1.0\linewidth]{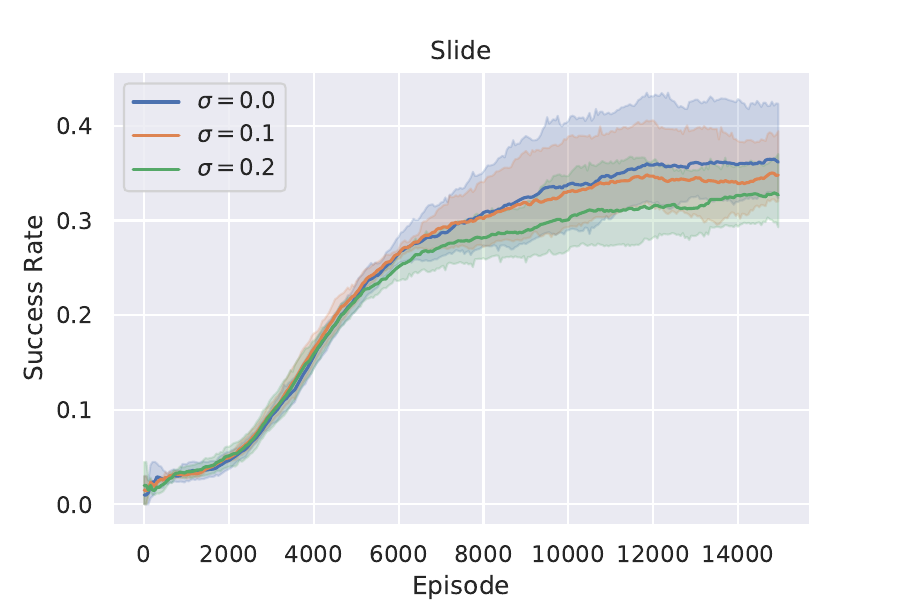}
\end{minipage}%
\begin{minipage}[htbp]{0.33\linewidth}
\centering
\includegraphics[width=1.0\linewidth]{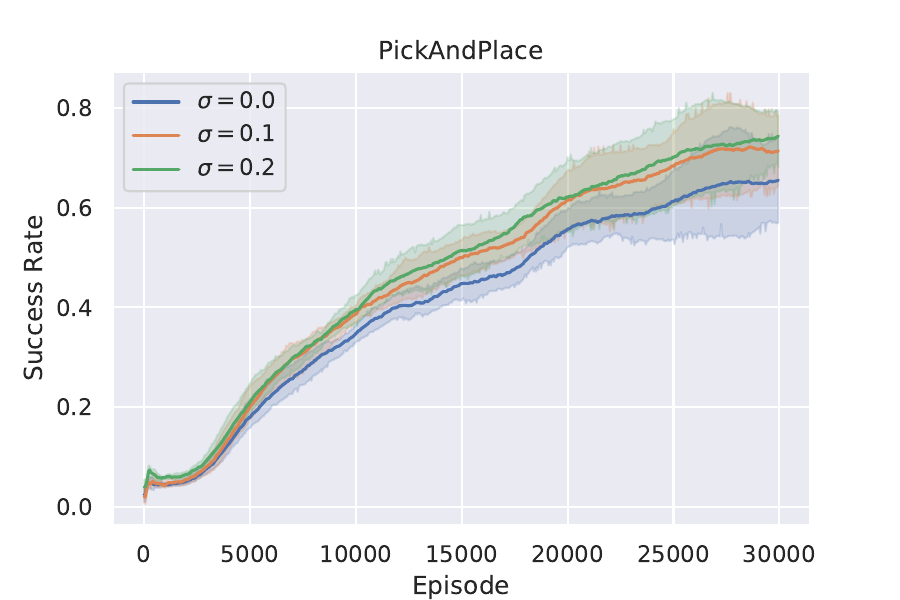}
\end{minipage}%
\caption{The test success rate comparison on the FetchPush-v1, FetchSlide-v1 and FetchPickAndPlace-v1 with different scale of noise applied in policy evaluation. The results are generated with 5 different random seeds}
\label{eval_exp_factor}
\end{figure*}

As we have shown in the numerical simulation, the bias of learned deterministic policy reduces the success rate. Such bias can be attributed to the extrapolation error~\citep{fujimoto2018off}.
Consequently, we introduce a Gaussian noise term in the learned policy to form a stochastic policy for robustness. Our ablation studies on the selection of different scales of such noise terms are shown in Fig.\ref{eval_exp_factor}. The results are generated with 5 different random seeds, showing proper noise terms can help to overcome the extrapolation error and therefore improve the evaluation performance.
It worths noting that applying larger noise in the game of FetchSlide will lead to performance decay, as the game relies on precise manipulation: after the robotic arm hitting the block, the block will become out of reach, and therefore the agent can not correct the error anymore.

\end{document}